%% file: ms.tex
\documentclass{article}

\input{packages}
\input{mydef}

\title{On the properties of some low-parameter models for color reproduction in terms of spectrum transformations and coverage of a color triangle}
\author[1]{Alexey Kroshnin}
\author[2, 1, *]{Viacheslav Vasilev}
\author[1, 2]{Egor Ershov}
\author[1]{\\Denis Shepelev}
\author[1]{Dmitry Nikolaev}
\author[3]{Mikhail Tchobanou}
\affil[1]{Institute for Information Transmission Problems, Moscow, Russia}
\affil[2]{Moscow Institute of Physics and Technology, Moscow, Russia}
\affil[3]{Huawei Technologies, Moscow, Russia}
\affil[*]{\href{mailto:vasilev.va@phystech.edu}{vasilev.va@phystech.edu}}
\date{}

\begin{document}
\maketitle

\begin{abstract}
One of the classical approaches to solving color reproduction problems, such as color adaptation or color space transform, is the use of low-parameter spectral models. 
The strength of this approach is the ability to choose a set of properties that the model should have, be it a large coverage area of a color triangle, an accurate description of the addition or multiplication of spectra, knowing only the tristimulus corresponding to them.
The disadvantage is that some of the properties of the mentioned spectral models are confirmed only experimentally.
This work is devoted to the theoretical substantiation of various properties of spectral models.
In particular, we prove that the banded model is the only model that simultaneously possesses the properties of closure under addition and multiplication.
We also show that the Gaussian model is the limiting case of the von Mises model and prove that the set of protomers of the von Mises model unambiguously covers the color triangle in both the case of convex and non-convex spectral locus.
\end{abstract}

\input{sections/introduction}
\input{sections/model.tex}
\input{sections/spectral_model.tex}

\input{sections/param.tex}
\input{sections/conclusion.tex}

%\printbibliography
\bibliographystyle{ieeetr}
\bibliography{biblio}

\begin{appendices}
\input{sections/appendix}
\end{appendices}

\end{document}

%% file: packages.tex
\usepackage{authblk}

\usepackage[T2A]{fontenc}
\usepackage[utf8]{inputenc}
\usepackage[russian, english]{babel}
\usepackage{amsthm,amsmath,amssymb,amsfonts}
\usepackage{mathtools}
\usepackage{esint}
\usepackage[colorlinks,citecolor=blue,urlcolor=blue]{hyperref}
\usepackage[linesnumbered,ruled,vlined]{algorithm2e}
\usepackage{mathrsfs}
\usepackage{graphicx}
\usepackage{bbm}
\usepackage{bm}
\usepackage{xifthen}
\usepackage{comment}
\usepackage[dvipsnames]{xcolor}
\usepackage[title]{appendix}
\usepackage{hyperref}

\numberwithin{equation}{section}
\theoremstyle{plain}
\newtheorem{theorem}{Theorem}[section]
\newtheorem*{theorem*}{Theorem}
\newtheorem{lemma}[theorem]{Lemma}

\newtheorem{definition}{Definition}
\newtheorem{proposition}[theorem]{Proposition}

\newtheorem{remark}{Remark}
\newtheorem{appendix_remark}{Remark A.\ignorespaces}

\usepackage{amsmath,amsthm}
\usepackage{mathtools}
\usepackage{amssymb}

\usepackage{bbm}
\usepackage{bm}

\usepackage[shortlabels]{enumitem}
\setlist[enumerate]{noitemsep}

%% file: mydef.tex
\newcommand*{\eqset}{\coloneqq}

\DeclareMathOperator{\supp}{supp}
\DeclarePairedDelimiter{\norm}{\lVert}{\rVert}
\DeclarePairedDelimiter{\abs}{\lvert}{\rvert}

\newcommand*{\R}{\mathbb{R}}
\let \C \relax
\newcommand*{\C}{\mathbb{C}}
\newcommand*{\N}{\mathbb{N}}
\newcommand*{\Z}{\mathbb{Z}}

\newcommand*{\MM}{\mathcal{M}}
\newcommand*{\PP}{\mathcal{P}}
\newcommand*{\TT}{\mathcal{T}}
\let \SS \relax
\newcommand*{\SS}{\mathcal{S}}
\newcommand*{\RR}{{\mathcal R}}

\DeclareMathOperator{\conv}{conv}
\DeclareMathOperator{\cone}{cone}
\DeclareMathOperator{\Span}{span}
\DeclareMathOperator{\inter}{int}

\DeclareMathOperator{\ind}{\mathbbm{1}}

\newcommand*{\eps}{\varepsilon}
\newcommand*{\lmax}{\lambda_{\max}}
\newcommand*{\lmin}{\lambda_{\min}}
\newcommand*{\one}{\bm{1}}
\let \d \relax
\newcommand*{\d}{\,\mathrm{d}}

\renewcommand*{\bar}[1]{\overline{#1}}
\renewcommand*{\tilde}[1]{\widetilde{#1}}

\let \oldchi \chi
\renewcommand*{\chi}{{\bm \oldchi}}
\let \oldeta \eta
\renewcommand*{\eta}{{\bm \oldeta}}
\let \c \relax
\newcommand*{\c}{\bm{c}}

\let\emptyset\varnothing

%% file: sections/introduction.tex
\section{Introduction}\label{ref:introduction}

When analyzing images, information about the color of a scene element is usually available only from a three-channel sensor point of view. 
Due to the mismatch between the dimensions of the spectrum space and the sensor space, the measured tristimulus can be generated by different metameric spectra, which complicates the solution of the color reproduction problem \cite{horn1984exact}. 
As a result, lighting compensation is always possible only with some approximation \cite{brill1978device, west1979, qiu2018image}. 
The situation is similar with color space transform \cite{Finlayson1994} and spectral reconstruction \cite{otsu2018}.

To overcome this difficulty, one can introduce restrictions on the set of possible spectra \cite{logvinenko2009object, logvinenko2013metamer}, which allows us to construct correct color transformations.
A particular case of this approach relies on the so-called \textit{spectral models} that define a set of protomeric spectra \cite{Stiles1962, weinberg1976geometry, nikolayev1985model, Griffin2019}.
By the term ``spectral model'' we will mean bijective mapping of the tristimulus space into the space of model parameters. In this case, each spectrum from a certain subspace of all spectra is represented in a unique way using the model parameters \cite{nikolaev2007spectral}.

In the paper \cite{nikolaev2008spectral}, the authors note that in order to achieve high-quality color reproduction, the spectral model must have certain properties that correspond to the linear color formation model \cite{nikolaev2004linear}.
Particularly, it is important that the set of spectral approximations is closed under addition, multiplication by a number and multiplication by itself.
Closure under addition allows us to describe the summation effect of several lighting sources correctly, and closure under multiplication allows us to describe multiple reflections of light from surfaces \cite{nikolaev2006efficiency, gusamutdinova2017verification} and to approximate high-saturation spectra higher accuracy \cite{nikolaev2008spectral}.
Also, an important property is the color set coverage \cite{nikolaev2008spectral, mizokami2012, mirzaei2014object}.
As a rule, it is not always possible to satisfy all the properties simultaneously, so the question arises: ``How many (and which) properties can a single model have?''.

In the paper we discuss \textit{linear models} \cite{brill1978device, Stiles1962, yilmaz1962theory, land1971lightness, nyberg1971II, Sallstrom1973, buchsbaum1980spacial, cohen1982r_matrix, Maloney86constancy, maloney1986evaluation, Marimont1992, Lee1995}. 
They are of interest because, on the one hand, they are computationally efficient, and, on the other hand, they are closed under addition.
However, all of them have a significant drawback~-- they poorly approximate high saturation spectra \cite{maloney1986evaluation, brill1986chromatic} and, as a result, have a low level of coverage of the color triangle (using three primaries, it is impossible to cover strictly convex color triangle) \cite{nikolaev2007spectral}.

It is important to note a special case of linear model~-- the \textit{banded spectral model}, in which the spectra are assumed piecewise constant \cite{Stiles1962, land1971lightness, nyberg1971II, maximov1984transformation}.
Previously, it was hypothesized that this is the only model that simultaneously has the properties of closure under addition and multiplication \cite{nikolaev2007spectral}.
In this paper, we provide a proof of this statement in Section \ref{subsec:zonal}.
In this regard, it makes no sense to look for other models closed under addition and multiplication, but there is still the question of covering a set of colors.

Later on, exponent-based spectral models were proposed, in which the parameters are arguments of an exponential function (in the future we will be interested in a narrower class of them, see Section \ref{sec:spectral_model}).
The use of such models makes it possible to approximate high saturation spectra and increase the coverage area.
First of them was the \textit{Gaussian model} \cite{weinberg1976geometry, nikolayev1985model} and its variants \cite{mizokami2012, Brill2002, MacLeod2003, logvinenko2013object}.
Like all exponential models, the Gaussian model has the property of closure under multiplication.
In addition, its parameters intuitively correspond to the color appearance characteristics \cite{nikolaev2007spectral, nikolaev2008spectral, mizokami2012}.
The peak of the Gaussian roughly corresponds to hue, the standard deviation corresponds to saturation, and the amplitude corresponds to brightness or lightness. 
The Gaussian model can be used to explain the Abney effect, which might be an indication that the human eye code the colors in a similar way \cite{Mizokami2006}.
Experiments \cite{nikolaev2006efficiency, nikolaev2005comparative} has confirmed that the quality of color constancy problem description is significantly higher for the Gaussian models compared to the linear \cite{Lee1995} ones.
The disadvantage of the Gaussian models is that they do not allow achieving uniform coverage of chromaticity values on the color triangle \cite{mizokami2012, mirzaei2014object}.

In an attempt to solve this problem, the authors of \cite{nikolaev2007spectral} has introduced the \textit{von Mises model}.
Authors experimentally proved that this model provides complete coverage for the standard observer color triangle.
At the moment, however, it is not theoretically proven whether the von Mises model allows to completely cover the color triangle of an arbitrary sensor (including ones with non-convex spectral locus).
In the Section~\ref{sec:spectral_model}, the von Mises model is formally defined, in the Section~\ref{subsec:von_Mises} its connection with the Gaussian model is studied, and the following Section~\ref{sec:param} is devoted to the answers to the questions about coverage.

%% file: sections/model.tex
\section{Mathematical model and assumptions}\label{sec:model}

By $\lambda \in \Lambda = (0, + \infty)$, we denote a wavelength of light. 
To describe a spectral power distribution (SPD) of incident radiance on a sensor, we use a finite Borel measure $\mu$ on $\Lambda$: $\mu \in \MM_+(\Lambda)$.
Note that we use general measures, not just absolutely continuous (i.e.\ having density) w.r.t.\ the Lebesgue measure, because $\MM_+(\Lambda)$ contains discrete measures, which allows us to describe a laser, a gas-discharge lamp, etc. 
Moreover, it immediately indicates possible operations on SPDs: one can integrate a (response) function w.r.t.\ an SPD, multiply it by a function (e.g.,\ a reflectance), or sum up SPDs, but cannot multiply them.
These properties are sufficient for linear model description \cite{nikolaev2004linear}.
This also highlights the different nature of illuminants, represented by SPDs, and spectral reflectances.
We discuss the structure and properties of various SPDs families in more detail in Sections~\ref{sec:spectral_model} and \ref{sec:param}.

Now let us describe a model of color perception we use.
Following \cite{weinberg1976geometry}, we characterize a sensor (of a camera or an observer) by its \emph{response function} $\chi \colon \Lambda \to \R_+^d$.
Then, a \emph{color} $\c$ observed by the sensor under an incident light with the SPD $\mu \in \MM_+(\Lambda)$ is given by
\begin{equation}\label{eq:color_map}
    \c(\mu) \eqset \int_\Lambda \chi(\lambda) \d \mu(\lambda) \in \R_+^d.
\end{equation}
Note that one can consider this as an equivalence class of indistinguishable SPDs for a given sensor.
Any SPD corresponding to some color $\c$ is called a \emph{metamer} of $\c$ \cite{weinberg1976geometry}.
We assume $\chi \in C_b(\Lambda; \R_+^d)$, i.e.\ it is a continuous bounded vector function.
Also, suppose $\supp \chi = [\lmin, \lmax]$ with $0 \le \lmin < \lmax < \infty$ depending on $\chi$ and $\chi(\lambda) \neq \bm{0}$ on $(\lmin, \lmax)$.

\begin{figure}
    \center{\includegraphics[scale=0.4]{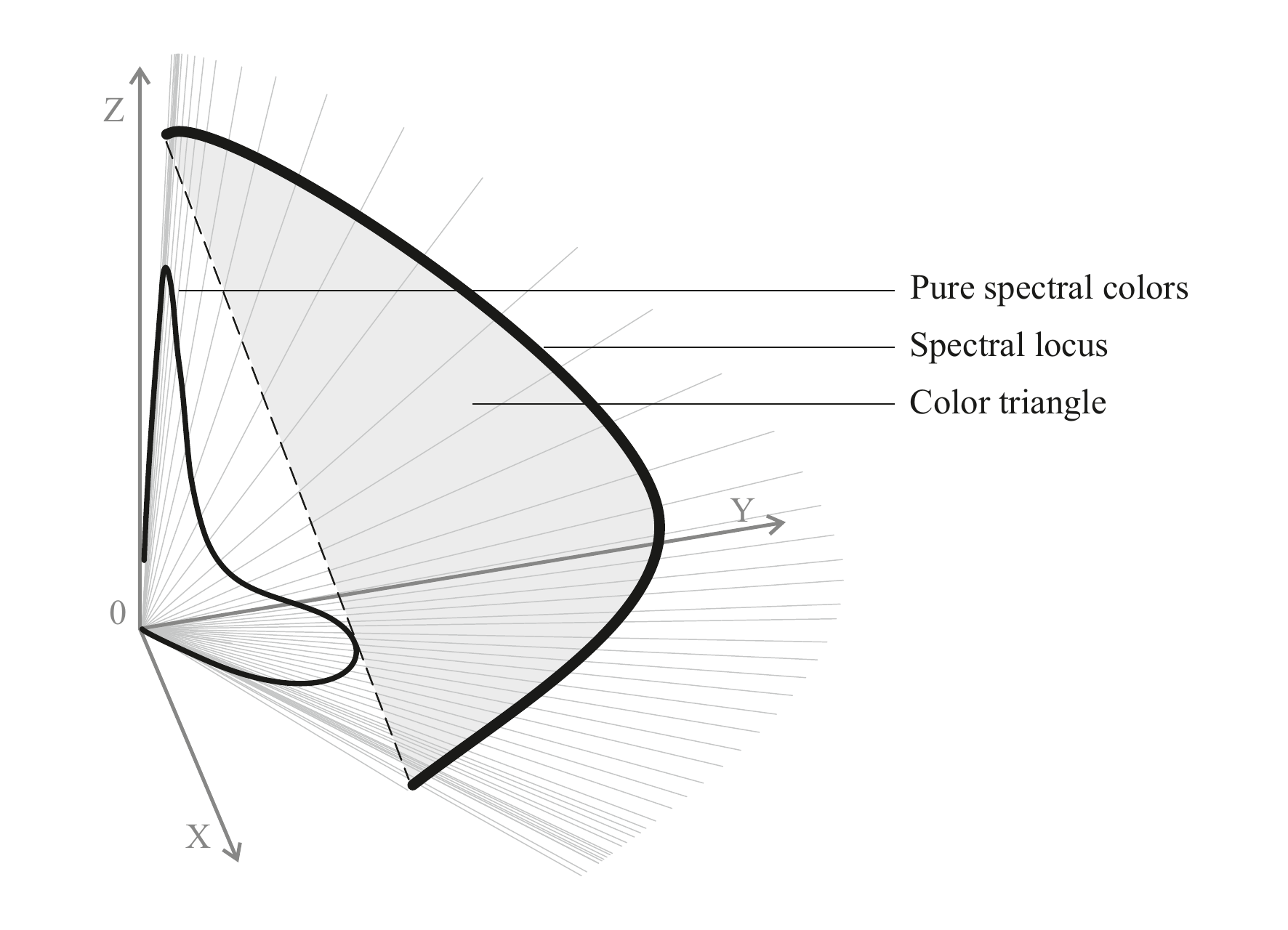} \includegraphics[scale=0.4]{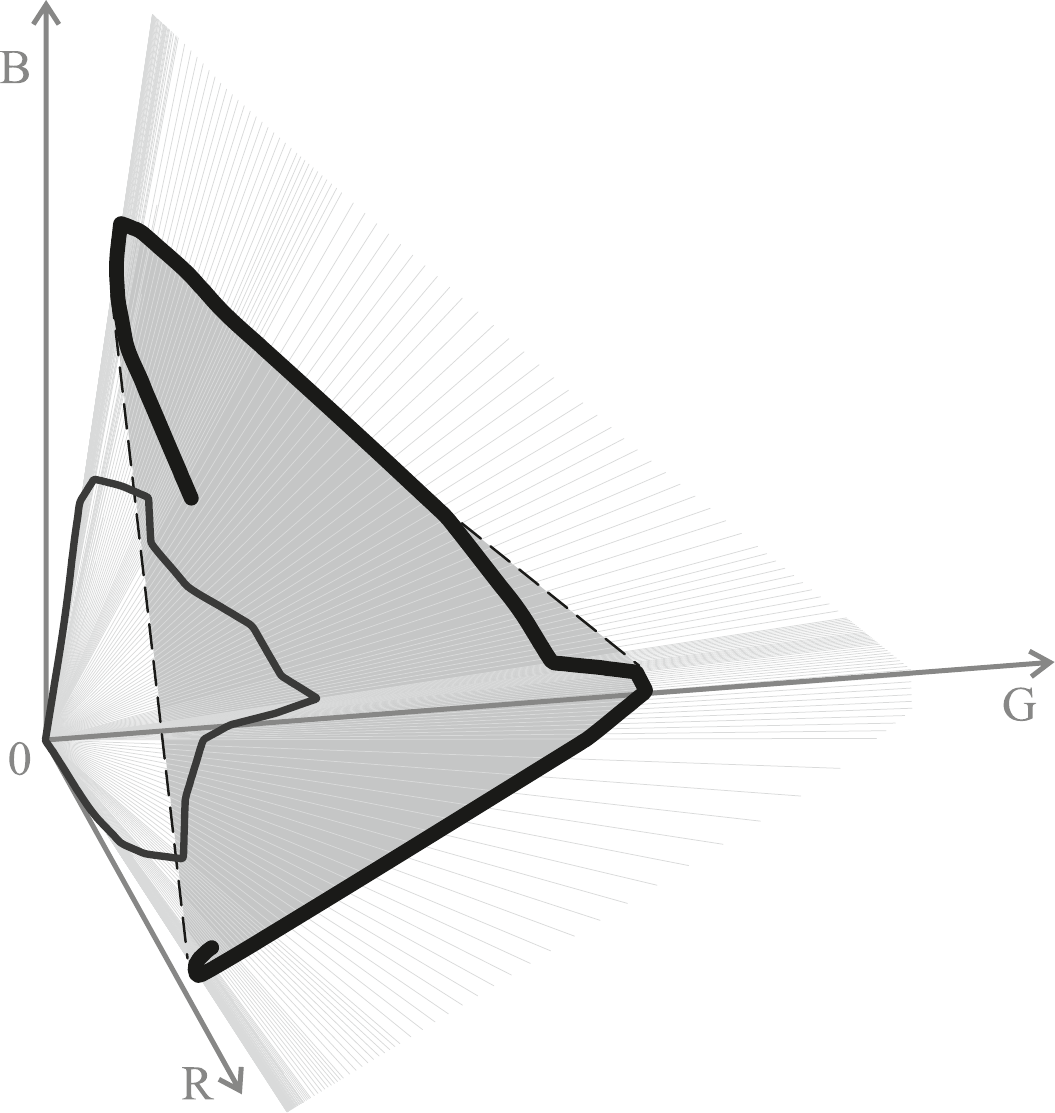}}
    \caption{\centering Color coordinates, spectral locus, and color triangle for the standard  observer CIE 1931 and Nikon D90 (provided by \cite{Jiang2013}). In the case of a camera, a non-convexity of the locus can be observed.} \label{fig:sensor_color_space}
\end{figure}

Now we define the normalized response function 
\begin{equation}\label{eq:eta}
    \eta(\lambda) \eqset \frac{\chi(\lambda)}{\langle \one, \chi(\lambda) \rangle} \in \Delta^{d - 1},
\end{equation}
where $\one$ is the vector of ones, $\Delta^{d - 1} \subset \R^d$ is the standard \mbox{$(d-1)$-dimensional} simplex, and $\langle \cdot, \cdot \rangle$ stands for the dot product.
We assume that $\eta$ can be continuously extended to $[\lmin, \lmax]$, i.e.\ there exist
\[
\eta(\lmin) = \lim_{\lambda \to \lmin + 0} \eta(\lambda), \; \eta(\lmax) = \lim_{\lambda \to \lmax - 0} \eta(\lambda).
\]
The curve $\eta([\lmin, \lmax])$ is called the \emph{spectral locus}, see Fig.~\ref{fig:sensor_color_space}.
Then we define a reweighted SPD as
\begin{equation}\label{eq:mu_tilde}
    \tilde\mu \eqset \langle \chi(\cdot), \one \rangle \mu \in \MM_+(\Lambda).
\end{equation}
Thus,
\[
\c(\mu) = \int_{[\lmin, \lmax]} \eta \d \tilde\mu.
\]
However, we would like to point out that $\tilde\mu$ is concentrated on $(\lmin, \lmax)$, and not all measures from $\MM_+([\lmin, \lmax])$ can be represented in this way (e.g., the ones having an atom at $\lmin$ or $\lmax$).

It is easy to see that the set of all colors for a given sensor is a pointed convex cone $\PP \subset \R_+^d$ called the \emph{color cone}. 
Take an affine subspace $A = \left\{\bm{x} \in \R^d : \langle \bm{x}, \one \rangle = 1\right\}$; then the \emph{color triangle} $\TT \eqset \PP \cap A$ is a base of $\PP$, i.e. 
\[
\forall \c \in \PP\; \exists a \ge 0, \bm{u} \in \TT \text{ such that } \c = a \bm{u}. 
\]
We will usually consider $\TT$ as a subset of the hyperplane $A$ endowed with the corresponding topology.
We always suppose $\PP$ is non-degenerate, i.e.\ $\inter \PP \neq \emptyset$ (hence $\inter \TT \neq \emptyset$).

Finally, let us briefly discuss a relation between the spectral locus and the color triangle.

\begin{lemma}\label{lem:convex_hull}
    The following relations hold between the color triangle and the spectral locus:
    \[
    \inter \conv\bigl(\eta([\lmin, \lmax])\bigr)
    \subset \TT \subset \conv\bigl(\eta([\lmin, \lmax])\bigr).
    \]
\end{lemma}

\begin{proof}
    Since $\eta$ is continuous, for any compact set $K \subset [\lmin, \lmax]$ it holds that $\eta(K) \eqset \left\{\eta(\lambda) : \lambda \in K\right\}$ is compact, thus its convex hull $\conv\bigl(\eta(K)\bigr)$ is compact as well due to Carath{\'e}odory's theorem. 
    By the definition of the color cone and the color triangle we have
    \begin{align*}
        \TT &= \left\{\c(\mu) : \mu \in \MM_+(\Lambda),\; \int_\Lambda \langle \one, \chi(\lambda) \rangle \d \mu(\lambda) = 1\right\} \\
        &= \left\{\int_{[\lmin, \lmax]} \eta \d \tilde\mu : \mu \in \MM_+(\Lambda),\; \tilde\mu([\lmin, \lmax]) = 1\right\} .
    \end{align*}
    For any $\mu \in \MM_+(\Lambda)$ with $\tilde\mu([\lmin, \lmax]) = 1$ one can find a sequence of discrete measures $\mu_n \in \MM_+(\Lambda)$ such that $\mu_n$ weakly converge to $\mu$ and $\tilde\mu_n([\lmin, \lmax]) = 1$. 
    
    Since
    \begin{align*}
        \c(\mu_n) &\in \conv\bigl(\eta((\lmin, \lmax))\bigr) \\
        &= \left\{\sum_{i = 1}^n w_i \eta(\lambda_i) : \sum_{i = 1}^n w_i = 1,\, w_i \ge 0,\, \lambda_i \in (\lmin, \lmax),\, n \in \N\right\},
    \end{align*}
    $\c(\mu_n) \to \c(\mu) \in \TT$, and $\conv\bigl(\eta([\lmin, \lmax])\bigr)$ is closed, we conclude that 
    \[
    \TT \subset \conv\bigl(\eta([\lmin, \lmax])\bigr).
    \]
    On the other hand, taking again discrete measures $\mu$ one can obtain any color from $\conv\bigl(\eta((\lmin, \lmax))\bigr)$.
    Hence 
    \[
    \inter \conv\bigl(\eta([\lmin, \lmax])\bigr) 
    \subset \conv\bigl(\eta((\lmin, \lmax))\bigr) 
    \subset \TT.
    \]
\end{proof}

\begin{remark}
    If $\eta(\lmin)$ or $\eta(\lmax)$ are not physically reachable, then ``purple'' colors can be outside the color triangle, nevertheless, the color triangle coincides with the spectral locus convex hull up to the boundary.
\end{remark}

%% file: sections/spectral_model.tex
\section{Spectral models: properties and uniqueness}\label{sec:spectral_model}

We define a \emph{spectral model} as a parametric family of \emph{illuminants} $\SS \subset \MM(\Lambda)$ and spectral \emph{reflectances} $\RR$ that are Borel functions from $\Lambda$ to $\R$ \cite{nikolaev2007spectral, nikolaev2008spectral}. 
Note that from a physical point of view, we should consider only nonnegative measures $\SS \subset \MM_+(\Lambda)$ and reflectances making values only from $[0, 1]$; however, sometimes a general setting is considered, e.g.\ in linear models described below.
In the simplest case, a radiance on the sensor is given by multiplication of an illuminant $S \in \SS$ by a reflectance $R \in \RR$ and some positive constant $c > 0$ depending on the geometry of the scene and the viewing conditions: $\mu = c R S$.
The above model is by definition \emph{hybrid} in terms of \cite{nikolaev2008spectral}, i.e.\ the set of illuminants differs from the set of reflectances. However, it is possible to consider a non-hybrid model as well: fix a \emph{reference} SPD $\mu_0$ and consider illuminants of the form
\[
\SS = \left\{c R \mu_0 : R \in \RR,\; c \ge 0\right\}.
\]
Probably the most popular type of a (hybrid) spectral model is the linear one, where $\SS$ and $\RR$ are finite-dimensional linear spaces of signed measures and functions, respectively \cite{brill1978device, yilmaz1962theory, maloney1986evaluation}.
Another model that we are interested in is the \emph{von Mises} one (called the Besselian model in \cite{nikolaev2007spectral}) with width $\Delta \lambda > 0$: it is a non-hybrid model with spectral densities and reflectances of the form
\begin{equation}\label{eq:von_mises}
    f_{a,b,s}(\lambda) \eqset \exp\left(b + a \cos \frac{2 \pi (\lambda - s)}{\Delta \lambda}\right).
\end{equation}
Obviously, if $\Delta \lambda = \lmax - \lmin$, these functions are periodic on $[\lmin, \lmax]$.
We also define the \emph{generalized von Mises} family generated by a function $h$ as follows:
\begin{equation}\label{eq:gen_von_mises}
    f_{a,b,s}(\lambda) \eqset \exp\left(b + a h(\lambda - s)\right).
\end{equation}
In particular, taking $h(x) = -x^2$, we obtain the well-known Gaussian model (if we allow negative $a$, then it also includes reciprocal Gaussians) \cite{weinberg1976geometry}.

Now, let us list some important properties that a spectral model may satisfy (cf.\ \cite[Section 3]{nikolaev2008spectral}).
\begin{enumerate}
    \item\label{item:scalable_illum} $\SS$ is closed under multiplication by a positive constant.
    
    \item\label{item:scalable_refl} $\RR$ is closed under multiplication by a constant from $[0, 1]$ or $\R_+$.

    \item\label{item:mult} $\RR$ is closed under pointwise multiplication.
    
    \item\label{item:mult_prior} $\SS$ is closed under multiplication by reflectances.
    
    \item\label{item:additivity} Additivity: $\SS$ or $\RR$ is closed under addition.
    
    \item\label{item:complete} Completeness: map from $\SS$ to colors is surjective.
    
    \item\label{item:injective} Injectivity: map from $\SS$ to colors is one-to-one.
    
    \item\label{item:periodic} Periodicity: $\SS$ and $\RR$ are closed under cyclical shifts on some interval $[\lmin, \lmax]$.
\end{enumerate}

It is quite natural to always assume properties~\ref{item:scalable_illum} and~\ref{item:scalable_refl}, since a source SPD and a reflectance define an incident radiance only up to a multiplicative constant depending on the geometry of the scene and the viewing conditions.
Property~\ref{item:mult} allows incorporating multiple reflections. 
Property~\ref{item:mult_prior} is related to conjugate priors in Bayesian statistics and can be important to non-hybrid models.
The role of property~\ref{item:periodic} will be explained later in Section~\ref{sec:param}.

Finally, properties~\ref{item:complete} and~\ref{item:injective} allow us to parametrize colors via a spectral model. 
This, in turn, is essential for the solution of a color constancy or an illumination discounting problems in the following way (an \emph{inversion model} in terms of \cite{brill1986chromatic}): assume we simultaneously obtain colors corresponding to a source $\c(S)$ and a reflected light $\c(R S)$; then we can estimate an illuminant $S$ from the first color and, based on this estimate, a reflectance $R$ from the second one. Given the reflectance estimate, one can compute a corresponding color under a reference illumination, e.g.\ an equi-energy spectrum or daylight (D50, D65) spectra.
In particular, for a linear model, it boils down to solving two systems of linear equations: a system for $S$ is fixed, and the other one for $R$ depends on its solution.
Another possible way is to use a group of \emph{protomers} \cite{weinberg1976geometry}: consider a non-hybrid model where $\SS = \left\{R \mu_0 : R \in \RR\right\}$, and $\RR$ is a Lie group w.r.t.\ pointwise multiplication; we call them protomers if the map $\c \colon \SS \to \PP$ is injective, i.e.\ for any color $\c$ there exists at most one metamer from this family. As shown in \cite{weinberg1976geometry}, $\RR$ has a form
\begin{equation}\label{eq:loglinear}
    \RR = \left\{R(\lambda) = \exp\left(\sum_{k = 1}^3 p_k P_k(\lambda)\right) : p_k \in \R\right\}.
\end{equation}
Obviously, one can choose any basis $\{P_1, P_2, P_3\}$ that spans the same linear space.
By definition, this model satisfies properties \ref{item:mult}, \ref{item:mult_prior}, and \ref{item:injective}. 
Moreover, if $\Span\{P_1, P_2, P_3\}$ contains a constant function, it also satisfies properties \ref{item:scalable_illum} and \ref{item:scalable_refl}. 
An important example of protomers is given by \cite{weinberg1976geometry} is the Gaussian family. 
Section~\ref{sec:param} shows that under additional assumptions on a response function (satisfied by the standard observer's one), the von Mises model with $\Delta \lambda = \lmax - \lmin$ is protomeric and complete.
Note that the Gaussian family is not complete in the case of the standard observer \cite{nikolaev2007spectral, mizokami2012, mirzaei2014object}.

Further, in particular, it will be shown that, from the point of view of certain properties, the banded spectral model (see Section \ref{ref:introduction}) and the von Mises model are unique.

\input{sections/zonal}
\input{sections/msm}

%% file: sections/zonal.tex
\subsection{Banded spectral model}\label{subsec:zonal}

The following proposition shows that the banded model is the only model closed under both addition and multiplication.

\begin{proposition}
    Let $K \subset \R^\Lambda$ be a $k$-dimensional convex cone of functions from $\Lambda$ to $\R$, closed under pointwise multiplication, i.e.\ for any $f, g \in K$ it holds that $f g \in K$.
    Then there exist unique nonempty disjoint sets $A_1, \dots, A_k \subset \Lambda$ such that any function $f \in K$ can be represented as
    \[
    f(\lambda) = \sum_{i = 1}^k v_i \ind_{A_i}(\lambda), \; v_i \in \R.
    \]
    \label{prop:banded_unique}
\end{proposition}

\begin{proof}
    For the $k$-dimensional convex cone $K$, its linear span is
    \[
    L_K \eqset \Span(K) = K - K \eqset \left\{f - g : f, g \in K\right\},
    \]
    and $\dim(L_K) = k$. Thus, any function from $L_K$ can be represented as
    \[
    f(\lambda) = \sum_{i = 1}^k v_i b_i(\lambda), \; v_i \in \R,
    \]
    and all the basis functions $\{b_1(\lambda), \dots, b_k(\lambda)\}$ belong to the cone $K$. 
    
    \paragraph{Step 1.} We show that an arbitrary function from $L_K$ can take no more than $k$ non-zero values. 
    Assume the opposite, namely, that there is a function $f \in L_K$ taking $k + 1$ a non-zero value.
    Consider $\lambda_1, \dots, \lambda_{k+1}$ from $\Lambda$, such that $f(\lambda_1) = a_1, \dots, f(\lambda_{k+1}) = a_{k+1}$ with different values $a_1, \dots, a_{k+1} \in \R \setminus \{0\}$. 
    Consider the determinant of the matrix of powers:
    \[
    \begin{vmatrix}
        a_1 & a_1^2 & \cdots & a_1^{k+1}\\
        a_2 & a_2^2 & \cdots & a_2^{k+1}\\
        \vdots & \vdots & \ddots & \vdots\\
        a_{k+1} & a_{k+1}^2 & \cdots & a_{k+1}^{k+1}\\
    \end{vmatrix} = \prod^{k+1}_{i=1}a_i\prod^{k+1}_{i>j}(a_i-a_j).
    \]
    The column $j$ of the matrix of this determinant consists of the values of the function $f^j \in L_K$ at the points $\lambda_1, \dots, \lambda_{k+1}$.
    Since the values $a_1, a_2, \dots, a_{k+1}$ are distinct, the determinant turns out to be non zero.
    However, the dimension of the space is $k$, which means that the rank of the matrix does not exceed $k$. 
    Thus, the determinant must be zero, and we get a contradiction with the fact that functions can take more than $k$ values.
    
    \paragraph{Step 2.} Now we show that there is a function from $L_K$ that takes exactly $k$ non-zero values. Suppose the opposite, let there be a function $f$ that takes $m$ non-zero values $a_1, \dots, a_m$, where $m < k$, and there are no function from $L_K$ taking more values. It follows that there exist non-empty disjoint sets $A_1, \dots A_m \subset \Lambda$, such that $A_i = f^{-1}(a_i)$, $i \in 1, \dots, m$.
    
    Take a function $g \in L_K$. Obviously, any their linear combination $\alpha f + \beta g \in L_K$ also takes no more than $m$ non-zero values, and hence on any set $A_i$ these functions take constant values. This means that any function from $L_K$ can be represented as a linear combination of only $m$ basis functions of the form $b_i = \ind_{A_i}$, which contradicts the fact that the dimension of the space is $k$.
    
    Therefore, any function from $L_K$ and, consequently, from the cone $K$ is a simple function:
    \[
    f(\lambda) = \sum_{i = 1}^k v_i \ind_{A_i}(\lambda),
    \]
    where $v_i \in \R$, and sets $A_1, \dots A_k \subset \Lambda$ are such that $A_i \cap A_j = \emptyset \; \forall i \neq j$.
    
    \paragraph{Step 3.}
    Now we show uniqueness of the sets $A_1, \dots, A_k$. Suppose that there is another similar family of non-empty disjoint sets $B_1, \dots, B_k$. Let a function $f \in L_K$ take $k$ non-zero different values. If the sets $\{A_i\}_{i=1}^k$ and $\{B_j\}_{j=1}^k$ do not match, there exist a set $A_l$, $\lambda' \in B_p, \; p\in\{1,\dots,k\}$, and $ \lambda'' \in B_q, \; q \in \{1,\dots,k\} \setminus \{p\}$, such that $\lambda', \lambda'' \in A_l$.
    Then $f(\lambda') = f(\lambda'')$, which means that on $B_p$ and on $B_q$ the function takes the same values, which contradicts the fact that the function takes $k$ values.
\end{proof}

According to \cite{Shepelev2020}, for any banded model there exists a linear mapping from the model parameters space to the color space, represented as a matrix with columns from the color cone.
The converse is also true, i.e. that any such matrix approximately corresponds to some banded model.

\begin{proposition}
    Let $\tilde{\mu}$ be atomless and $\supp \tilde{\mu} = [\lmin, \lmax]$,  and $\bar{\TT}$ denotes the closure of $\TT$. Then for any matrix $P \in \R^{d \times n}$ with columns $\c_1, \dots, \c_n \in \bar{\TT}$ and $\eps > 0$ there exist disjoint closed sets $A_1, \dots, A_n \subset [\lmin, \lmax]$ such that
    \[
    \norm*{\c_i - \frac{1}{\tilde\mu(A_i)} \int_{A_i} \eta \d \tilde{\mu}} \le \eps, \quad 1 \le i \le n,
    \]
    i.e.\ $P$ is almost a matrix of transition from parameters of a banded model to the color space.
    \label{prop:banded_matrix}
\end{proposition}

\begin{proof}
    For any $\c_i$ there is $\c'_i \in \inter \TT$ and $f_i \ge 0$ such that $\norm{\c_i - \c'_i} \le \frac{\eps}{2}$, $\int_{[\lmin, \lmax]} f_i \d \tilde\mu = 1$, and $\c'_i = \c(f_i \mu) = \int_{[\lmin, \lmax]} f_i \eta \d \tilde\mu$
    (see Section~\ref{sec:param}). 
    Take $N \in \N$ and divide $[\lmin, \lmax]$ into $n N$ disjoint segments $I_{p, nN} = [a_{p-1, nN}, a_{p,nN})$, $p = \overline{1,nN}$, with $a_{p, n N} \eqset \lmin + \frac{p}{n N} (\lmax - \lmin)$. Assign disjoint sets of such segments to each $\c'_i$, so that the segments numbered $i + kn$, $k=\overline{0,N-1}$, correspond to $\c'_i$. Each such set of segments $I_{p,nN}$ will correspond to a set $A_i$: $A_i \subset \bigcup_{k = 0}^{N-1} I_{i + kn, nN}$.
    
    For each $i$, select $N$ segments $J_{i,k} \subset I_{i+kn, nN}$, $k=\overline{0,N-1}$, such that
    \[
    \int_{a_{kn,nN}}^{a_{kn+n,nN}} f_i d \tilde\mu = C_N \tilde\mu(J_{i,k}).
    \]
    Note that it is possible once 
    \[
    C_N \ge \max_{i,k} \frac{1}{\tilde{\mu}(I_{i+kn, nN})} \int_{a_{kn,nN}}^{a_{kn+n,nN}} f_i d \tilde\mu
    \]
    since $\tilde{\mu}$ is atomless.
    Define sets $A_i \eqset \bigcup_{k = 0}^{N-1} J_{i, k}$.
    It follows from the construction of $J_{i, k}$ that $\tilde\mu(A_i) = \frac{1}{C_N}$ for any $i \in \{1,\dots,n\}$.
    
    Let us introduce the following piecewise constant function:
    \[
    \eta_N \eqset \sum_{k=0}^{N-1} \eta_{k, N} \ind_{[a_{kn,nN}, a_{kn+n,nN})}, \quad \eta_{k, N} \eqset \eta(a_{kn, nN}).
    \]
    Since the function $\eta$ is uniformly continuous, $\eta_N$ converges to $\eta$ uniformly as $N \to \infty$.
    Using this approximation, we obtain for any $i \in \{1,\dots,n\}$ that
    \begin{align*}
        C_N \int_{A_i}\eta d \tilde\mu
        &= \sum_{k=0}^{N-1} \int_{J_{i, k}} C_N \eta d \tilde\mu 
        \approx \sum_{k=0}^{N-1} \int_{J_{i, k}} C_N \eta_N d \tilde\mu \\
        &= \sum_{k=0}^{N-1} C_N \tilde\mu(J_{i, k}) \eta_{k, N} 
        = \sum_{k=0}^{N-1} \int_{a_{kn,nN}}^{a_{kn+n,nN}} \eta_N f_i d \tilde\mu \\
        &= \int_{\lmin}^{\lmax} f_i \eta_N d \tilde\mu 
        \approx \int_{\lmin}^{\lmax} f_i \eta d \tilde\mu = \c'_i,
    \end{align*}
    which completes the proof.
\end{proof}

\begin{remark}
    Note that if some column $\bm{p}_i$ of $P$ does not belong to $\bar{\PP}$, then for any set $A_i \subset \Lambda$ and $v_i \ge 0$ 
    \[
    \norm*{\bm{p}_i - v_i \int_{A_i} \eta \d \tilde{\mu}} \ge d\left(\bm{p}_i, \bar{\PP}\right) \eqset \inf_{\c \in \bar{\PP}} \norm{\bm{p}_i - \c} > 0,
    \]
    i.e.\ $P$ cannot be approximated by a matrix of a banded model.
\end{remark}

%% file: sections/msm.tex
\subsection{Von Mises model}\label{subsec:von_Mises}

The next lemma shows that the von Mises model given by \eqref{eq:von_mises} is in some sense the only family of form \eqref{eq:loglinear} satisfying properties \ref{item:scalable_refl}, \ref{item:periodic} (thus it necessarily is a generalized von Mises family \eqref{eq:gen_von_mises}), and able to approximate spectral colors.

\begin{proposition}
    Let $\mathcal{F}$ be a generalized von Mises family with continuous periodic function $h \colon [\lmin, \lmax] \to \R$ having only one maximum point on $[\lmin, \lmax]$. Assume $\mathcal{F}$ is closed under pointwise multiplication. Then it is the von Mises family, i.e.\ one can take $h(\lambda) = \cos \frac{2 \pi \lambda}{\lmax - \lmin}$.
    \label{prop:mises_uniqueness}
\end{proposition}

\begin{proof}
    W.l.o.g.\ assume $\lmax - \lmin = 2 \pi$.
    Consider the Fourier series for $h$:
    \[
    h(\lambda) = c_0 + \sum_{n \in \Z} c_n e^{i n \lambda} \text{~~a.e.},
    \]
    where $c_n = \overline{c_{-n}} \in \C$.
    Closedness under pointwise multiplication implies that for any $a_1 > 0$, $s_1 \in [\lmin, \lmax]$ there are $a_2 \ge 0$, $s_2 \in [\lmin, \lmax]$, $b_2 \in \R$ such that
    \[
    h(\lambda) + a_1 h(\lambda - s_1) = a_2 h(\lambda - s_2) + b_2.
    \]
    Thus their Fourier series coincide:
    \[
    (1 + a_1) c_0 + \sum_{n \in \Z} (1 + a_1 e^{- i n s_1}) c_n e^{i n \lambda}
    = a_2 c_0 + b_2 + \sum_{n \in \Z} a_2 e^{- i n s_2} c_n e^{i n \lambda} \text{~~a.e.},
    \]
    i.e.\ $b_2 = (1 + a_1 - a_2) c_0 + b_1$ and 
    \[
    (1 + a_1 e^{- i n s_1}) c_n = a_2 e^{- i n s_2} c_n \text{~~for all~~} n \in \N.
    \]
    Note that the equation $\abs{1 + a_1 z} = a_2$ has at most two solutions $z \in \C$, and they are conjugated. Then taking $s_1$ such that $\frac{s_1}{\pi}$ is irrational, we conclude that there is at most one $n = n_0 \in \N$ such that $c_n \neq 0$. Thus
    \[
    h(\lambda) = c_0 + c_{n_0} e^{i n_0 \lambda} + \overline{c_{n_0}} e^{- i n_0 \lambda} = c_0 + A \sin(n_0 (\lambda - \lambda_0)).
    \]
    Since $h$ has one maximum point on $[\lmin, \lmax]$, we get $n_0 = 1$. The claim follows.
\end{proof}

The next proposition links the von Mises model to the Gaussian one.

\begin{proposition}
    Fix $\lmin$, $\lmax$, and $D > 0$. Denote by $\mathcal{F}_{\Delta \lambda, D}$ the following subfamily of the von Mises model with the width $\Delta \lambda > 0$ restricted to $[\lmin, \lmax]:$
    \[
    \mathcal{F}_{\Delta \lambda, D} \eqset \left\{f_{a,b,s} : 0 \le a \le D,\; b \le D \right\}, \text{~~where $f_{a,b,s}$ comes from \eqref{eq:von_mises}}.
    \]
    In the limit $\Delta \lambda \to \infty$ it converges to (reciprocal) Gaussians in a sense that for any $f \in \mathcal{F}_{\Delta \lambda, D}$ there are $\alpha, \beta, \gamma \in \R$ such that
    \[
    \max_{\lmin \le \lambda \le \lmax} \abs*{f(\lambda) - \exp\left(\alpha \lambda^2 + \beta \lambda + \gamma\right)} \le \eps(\Delta \lambda) \to 0 \text{~~as~~} \Delta \lambda \to \infty.
    \]
    \label{prop:mises_gauss}
\end{proposition}

\begin{proof}
    For simplicity consider $\lmin = 0$, $\lmax = 1$. Fix $\Delta \lambda > 1$ and denote $A \eqset \frac{2 \pi}{\Delta \lambda}$.
    Using the Taylor theorem at the point $\lambda = 0$ we obtain that
    \[
    \cos\left(\frac{2 \pi (\lambda - s)}{\Delta \lambda}\right) = \cos(A s) + A \lambda \sin(A s) - \frac{A^2 \lambda^2}{2} \cos(A s) + O\left(\left(\frac{\lambda}{\Delta \lambda}\right)^3\right).
    \]
    Therefore, on $[0, 1]$ we have
    \begin{align*}
        f_{a,b,s}(\lambda) &= \exp\left[b + a \cos(A s) + A a \lambda \sin(A s) - \frac{A^2 a \lambda^2}{2} \cos(A s) + O\left(\frac{a}{(\Delta \lambda)^3}\right)\right] \\
        &= \exp\left[\alpha \lambda^2 + \beta \lambda + \gamma\right] \exp\left[O\left(\frac{D}{(\Delta \lambda)^3}\right)\right],
    \end{align*}
    where $\alpha = - \frac{A^2 a}{2} \cos(A s)$, $\beta = A a \sin(A s)$, $\gamma = b + a \cos(A s)$. Finally, note that $\max_{0 \le \lambda \le 1} \left(\alpha \lambda^2 + \beta \lambda + \gamma\right) = O(D)$, and $\exp\left[O\left(\frac{D}{(\Delta \lambda)^3}\right)\right] = O(1)$ once $\Delta \lambda \to \infty$, thus the claim follows.
\end{proof}

%% file: sections/param.tex
\section{On parametrization of colors and coverage of a color triangle}\label{sec:param}

In this section we consider only the case of $d = 3$, what corresponds to a standard observer and most cameras.
For the sake of simplicity we assume $\lmin = 0$, $\lmax = 1$.

Having found out that only the banded model can simultaneously satisfy the properties of closure under multiplication and addition it makes no sense to look for any model with the same two properties and with a wider coverage of chromaticity values on a color triangle (including strictly convex case). 
Getting rid of closure under multiplication seems to be not promising~--- in this way it is impossible to reproduce saturated colors, which are characterized by sharpened spectra. 
Other way is to get rid of the closure under addition, which is characteristic of linear models, and consider models based on the exponential functions. 
Since the Gaussian model is the limiting case of the von Mises model, the main purpose of this section will be to show that the von Mises model completely covers the color triangle in both the case of convex and non-convex spectral locus.

Here we study the question of colors parametrization (modulo intensity). 
Namely, given a response function $\chi$ and a \emph{reference SPD} $\mu \in \MM_+(\Lambda)$ we want to describe conditions under which for some (parametric) family of nonnegative spectral densities $\mathcal{F} \subset L^1(\mu)$ the map
\begin{equation}\label{eq:c_mu}
    \c_\mu(f) \eqset \c(f \mu) = \int_{[\lmin, \lmax]} f \eta \d \tilde\mu = \int_{\Lambda} f \chi \d \mu \in \PP
\end{equation}
is onto, one-to-one, or a bijection to the color triangle $\TT$ (thus the same holds for multiples of $f \in \mathcal{F}$ and the color cone $\PP$).
In terms of spectral models it means completeness and injectivity (properties \ref{item:complete} and \ref{item:injective}) of the set of illuminants $\SS = \left\{f \mu : f \in \mathcal{F} \right\}$.
Note that in the general case a linear model with $\mathcal{F} = \cone\{f_1, \dots, f_n\} \eqset \left\{\sum_{i = 1}^n v_i f_i : v_i \ge 0\right\}$ and non-negative basis functions $f_i$, $i = \overline{1, n}$, is not complete: indeed, it covers only a polygonal subset of the color triangle.
In this section we consider some families of densities including generalized von Mises families which cover the color triangle under suitable assumptions on the map $\eta$.

\begin{remark}
    If $f \ge 0$ and $\int_{\Lambda} f \d \tilde\mu = 1$, then $\c_\mu(f) \in \conv\bigl(\eta([\lmin, \lmax])\bigr)$.
\end{remark}

\begin{proof}
    It immediately follows from the proof of Lemma~\ref{lem:convex_hull}.
\end{proof}

Let $\tilde\mu(\Lambda) = 1$ and $\supp \tilde\mu = [\lmin, \lmax]$, i.e.\ $\tilde\mu([a, b]) > 0$ for any $\lmin \le a < b \le \lmax$.
It will be useful to consider a $1$-dimensional torus $\mathbb{T} = \mathbb{T}(\lmin, \lmax)$ as $[\lmin, \lmax]$ with identified points $\lmin$ and $\lmax$. We also define a closed cyclic interval as follows:
\[
[\lmin, \lmax] \supset [a, b]_{\mathbb{T}} = \begin{cases}
    [a, b], & a \le b, \\
    [a, \lmax] \cup [\lmin, b], & a > b.
\end{cases}
\]
Note that it is a subset of $[\lmin, \lmax]$, not of $\mathbb{T}$. 
In a similar way we can define an open cyclic interval $(a, b)_{\mathbb{T}}$, and half-open intervals $[a, b)_{\mathbb{T}}$, $(a, b]_{\mathbb{T}}$.

Since $A = \left\{\bm{x} \in \R^3 : \langle \bm{x}, \one \rangle = 1\right\}$ is $2$-dimensional, we can define an angle $\angle(\c', \c)$ from $\c \in A$ to $\c' \in A$ fixing an arbitrary direction and orientation of $A$.

\begin{figure}
    \centering
    \includegraphics[scale=0.48]{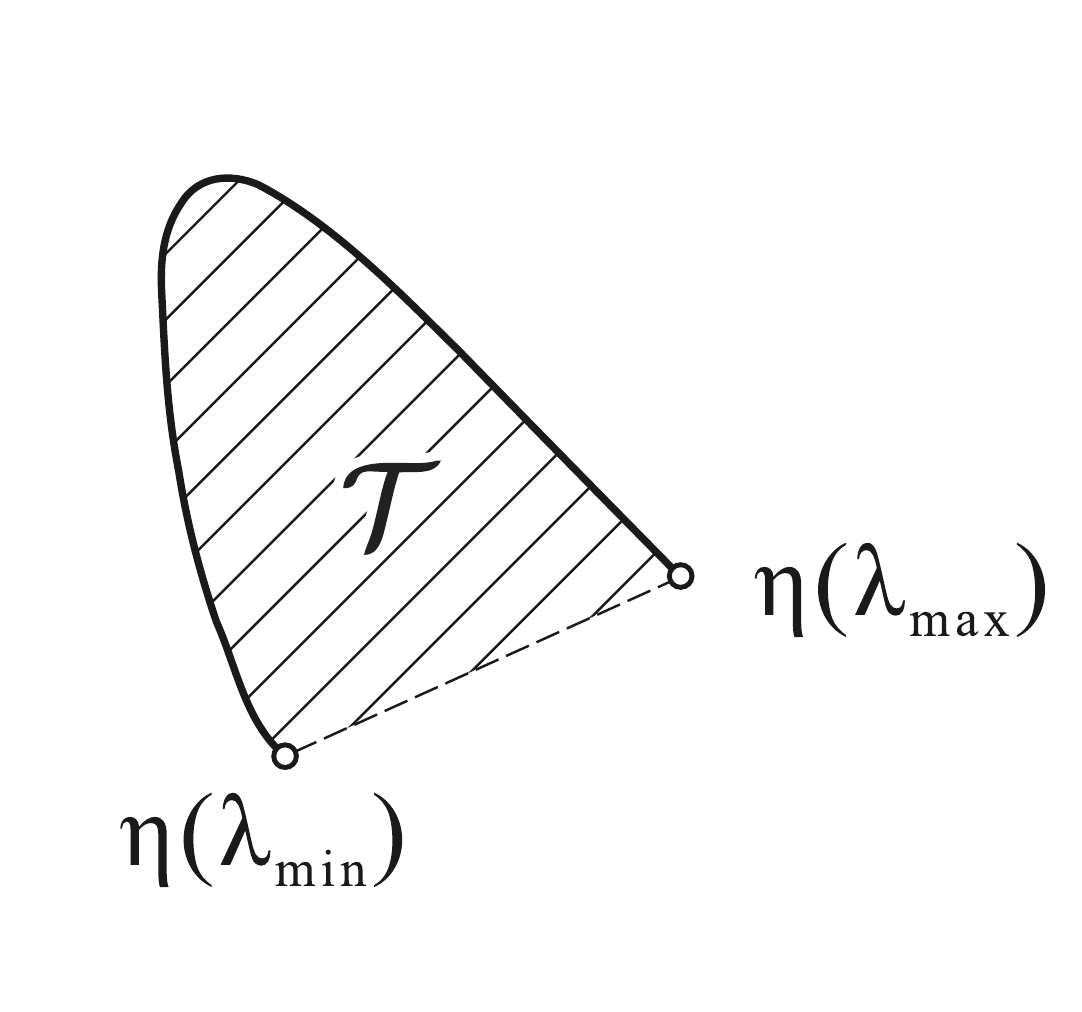}
    \includegraphics[scale=0.39]{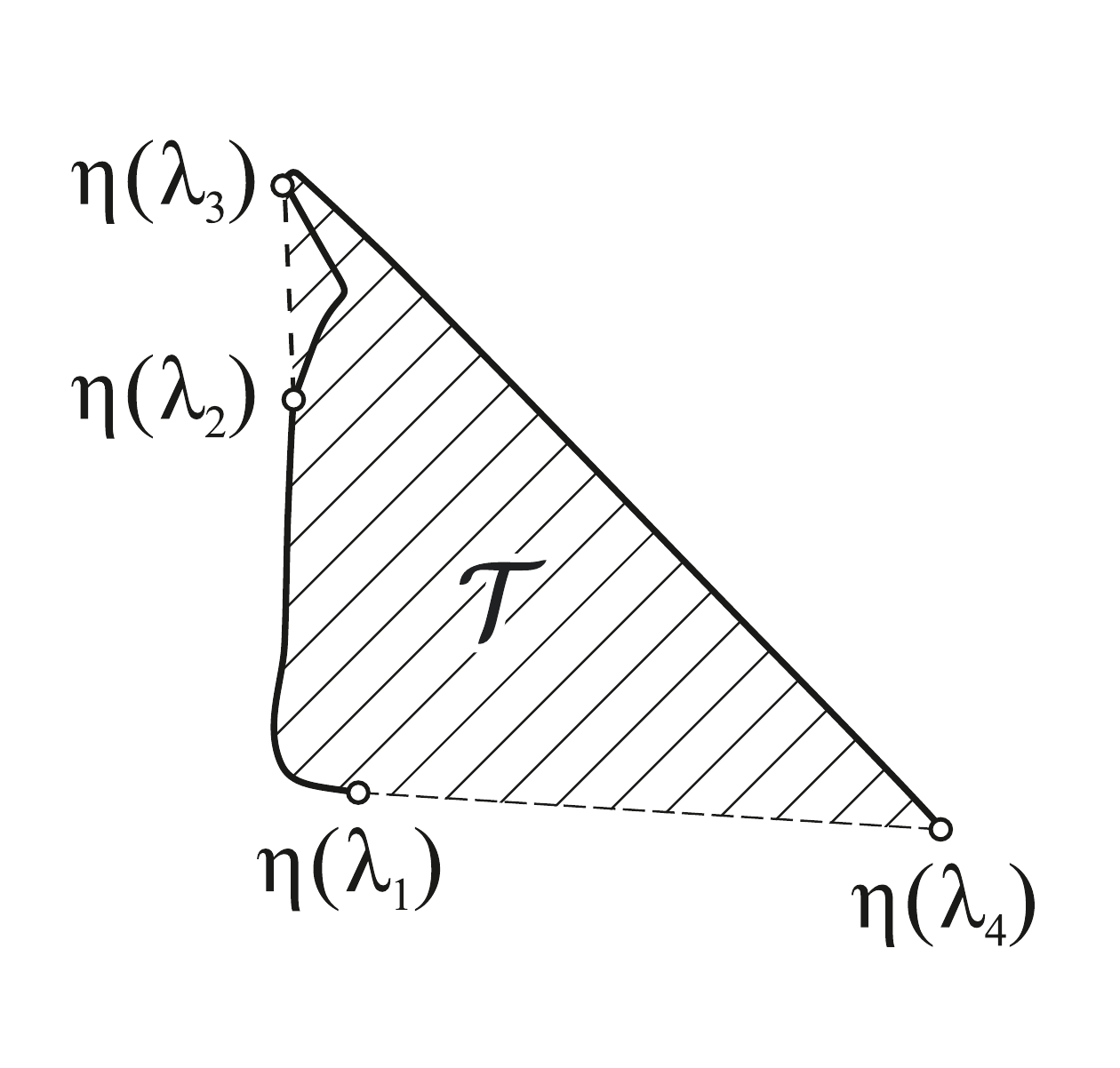}
    \caption{\centering Coverage using the von Mises model of the color triangle of a standard observer CIE 1931 and Nikon D90.}
    \label{fig:colortriangle}
\end{figure}

\begin{definition}\label{def:convex_locus}
    We say a spectral locus is convex if 
    \begin{equation}\label{eq:T_boundary}
        \eta([0, 1]) \subset \partial \TT,
    \end{equation}
    and for any $\c \in \inter \TT$ (a continuous version of) the angle from $\c$ to $\eta(\lambda)$ is a monotone function on $[0, 1]$ with $\abs*{\angle(\eta(0), \c) - \angle(\eta(1), \c)} \le 2 \pi$.
    
    If, in addition, $\eta(\lambda) \notin [\eta(a), \eta(b)]$ for any $a \neq b \neq \lambda \in [0, 1]$, then the spectral locus is strictly convex. Here $[\bm{x}, \bm{y}] \subset \R^3$ denotes the segment of the line connecting $\bm{x}$ and $\bm{y}$.
\end{definition}

It is important to note that the spectral locus for the standard observer CIE 1931 (Fig. \ref{fig:colortriangle}) is considered to be convex \cite{logvinenko2009object}.
Now we discuss some important properties of a convex or strictly convex spectral locus.

\begin{lemma}\label{lem:convex_locus_halfplane}
    If the spectral locus is convex, then for any closed half-plane $H \subset A$ there is a (closed) cyclic interval $I_H \subset [0, 1]$ such that
    \[
    \eta^{-1}(H) \eqset \bigl\{\lambda \in [0, 1] : \eta(\lambda) \in H\bigr\} = I_H.
    \]
    Moreover, $\partial \TT = \eta([0, 1]) \sqcup (\eta(0), \eta(1))$.
\end{lemma}

\begin{proof}
    Fix $\c \in \inter \TT$.
    W.l.o.g.\ assume that $\angle(\eta(0), \c) = 0$, $\angle(\eta(1), \c) = \theta$.
    Since $\TT$ is convex and bounded there exists a continuous bijection $\bm{b} \colon 2 \pi \mathbb{T} \to \partial \TT$ (here $2 \pi \mathbb{T}$ denotes the interval $[0, 2 \pi]$ with identified endpoints) such that $\angle(\bm{b}(\varphi), \c) = \varphi$ for any $\varphi \in 2 \pi \mathbb{T}$. Note that the inverse of $\bm{b}$ is also continuous.
    
    It is easy to see from a geometrical consideration that due to convexity of $\TT$ for any closed half-plane $H \subset A$ the set $\left\{\varphi \in [0, 2 \pi] : \bm{b}(\varphi) \in H\right\}$ is a closed cyclic interval $J_H$. Then intersection $J_H \cap [0, \theta]$ is also a closed cyclic interval in $[0, \theta]$, and we get from the monotonicity of $\angle(\eta(\lambda), \c)$ that
    \[
    \eta^{-1}(H) = \bigl\{\lambda \in [0, 1] : \angle(\eta(\lambda), \c) \in J_H\bigr\}
    \]
    is a closed cyclic interval in $[0, 1]$.
    
    Now note that $\theta > \pi$, otherwise $\eta([0, 1])$, and hence $\TT$, is contained in a convex circular segment with a center at $\c$ thus $\c \notin \inter \TT$.
    Suppose $\eta(0) \neq \eta(1)$ and consider the segment $(\eta(0), \eta(1))$. Since $\theta = \angle(\eta(1), \c) - \angle(\eta(0), \c) > \pi$ we obtain that $\c \notin (\eta(0), \eta(1))$, and thus 
    \begin{equation}\label{eq:segment_angles}
        \left\{\angle(\c', \c) : \c \in (\eta(0), \eta(1))\right\} = (\angle(\eta(1), \c), \angle(\eta(0) + 2 \pi) = (\theta, 2 \pi).
    \end{equation}
    Let $H$ be the closed half-plane containing $\c$ and with $\partial H$ passing through $\eta(0)$, $\eta(1)$. Assume that $\TT \setminus H \neq \emptyset$. Then there is $\lambda \in (0, 1)$ such that $\eta(\lambda) \in \TT \setminus H$. But then either $\angle(\eta(\lambda), \c) < \pi$ and 
    \[
    \eta(1) \in \inter \conv\{\c, \eta(0), \eta(\lambda)\} \subset \inter \TT,
    \]
    or $\angle(\eta(\lambda), \c) > \theta + \pi$ and 
    \[
    \eta(0) \in \inter \conv\{\c, \eta(1), \eta(\lambda)\} \subset \inter \TT.
    \]
    This contradicts to the fact that $\eta([0, 1]) \subset \partial \TT$.
    Therefore, $\partial H$ is a supporting line of $\TT$, and thus $(\eta(0), \eta(1)) \subset \partial \TT$. Finally, \eqref{eq:segment_angles} yields that $\eta([0,1]) \cap (\eta(0), \eta(1)) = \emptyset$.
    The claim follows.
\end{proof}

\begin{lemma}
    If the spectral locus is convex, then for any $a < b \in [0, 1]$ 
    \[
    \conv\bigl(\eta([a, b]_{\mathbb{T}})\bigr) \cap \conv\bigl(\eta([b, a]_{\mathbb{T}})\bigr) = [\eta(a), \eta(b)].
    \]
\end{lemma}

\begin{proof}
    Fix $a < b \in [0, 1]$ and assume there is 
    \[
    \c \in \Bigl[\conv\bigl(\eta([a, b]_{\mathbb{T}})\bigr) \cap \conv\bigl(\eta([b, a]_{\mathbb{T}})\bigr)\Bigr] \setminus [\eta(a), \eta(b)].
    \]
    Take a half-plane $H$ such that $c \notin H$, $\eta(a), \eta(b) \in H$ and a corresponding cyclic interval $I_H$ from Lemma~\ref{lem:convex_locus_halfplane}. Since $a, b \in H$ we have either $[a, b]_{\mathbb{T}} \subset I$ or $[b, a]_{\mathbb{T}} \subset I_H$, thus
    \[
    \conv\bigl(\eta([a, b]_{\mathbb{T}})\bigr) \cap \conv\bigl(\eta([b, a]_{\mathbb{T}})\bigr) \subset \conv\bigl(\eta(I_H)\bigr) \subset H.
    \]
    Therefore, $\c \in H$, and we got a contradiction.
\end{proof}

Let us also mention a simple sufficient condition for convexity of a spectral locus.
Let $\eta \in C^1([0,1])$. Denote by $\varphi(\bm{x})$ the directional angle of $\bm{x} \in \mathring{A} \eqset A - (1, 0, 0)$. If the function $\varphi(\eta'(\lambda))$ is increasing and 
\[
\varphi(\eta(1) - \eta(0)) \le \varphi(\eta'(0)) \le \varphi(\eta'(1)) \le \varphi(\eta(1) - \eta(0)) + 2 \pi,
\]
then the spectral locus is convex.

We say that a $\tilde\mu$-measurable function $f$ on $[0, 1]$ changes sign twice on $[0, 1]$, if $\tilde\mu\left(\{f > 0\}\right) > 0$, $\tilde\mu\left(\{f < 0\}\right) > 0$, and there exists a cyclic interval $I \subset [0, 1]$ such that both $I$ and $[0, 1] \setminus I$ have nonempty interior, $f \ge 0$ on $I$, and $f \le 0$ outside $I$.
The next lemma is useful to show that $\c_\mu$ on some families is one-to-one.

\begin{lemma}\label{lem:injective}
    Let the spectral locus be strictly convex. Take $\tilde\mu$-integrable functions $f, g$ such that $\int f \d \tilde\mu = \int g \d \tilde\mu = 1$ and $f - g$ changes sign twice on $[0, 1]$. Then $\c_\mu(f) \neq \c_\mu(g)$.
\end{lemma}

\begin{proof}
    Note that
    \[
    \c_\mu(f) - \c_\mu(g) = \int_{[0,1]} (f - g) \eta \d \tilde\mu
    = \int_{[0,1]} (f - g)_+ \eta \d \tilde\mu - \int_{[0,1]} (f - g)_- \eta \d \tilde\mu.
    \]
    Then $\c_\mu(f) = \c_\mu(g)$ is equivalent to
    \[
    \int_{[0,1]} (f - g)_+ \eta \d \tilde\mu = \int_{[0,1]} (f - g)_- \eta \d \tilde\mu.
    \]
    Define $w \eqset \int_{[0,1]} (f - g)_+ \d \tilde\mu = \int_{[0,1]} (f - g)_- \d \tilde\mu > 0$.
    Now take an interval $I \subset [0, 1]$ and $0 \le a < b \le 1$ such that $\overline{I} = [a, b]$ and (w.l.o.g.) $f \ge g$ on $I$, $f \le g$ outside (it exists by the assumption of the lemma). Then in the same way as in the proof of Lemma~\ref{lem:convex_hull} we obtain
    \begin{align*}
        \frac{1}{w} \int_{[0,1]} (f - g)_+ \eta \d \tilde\mu 
        &= \int_{I} \frac{(f - g)_+}{w} \eta \d \tilde\mu \in \conv\bigl(\eta([a, b]_{\mathbb{T}})\bigr) \\
        \frac{1}{w} \int_{[0,1]} (f - g)_- \eta \d \tilde\mu 
        &= \int_{[0, 1] \setminus I} \frac{(f - g)_-}{w} \eta \d \tilde\mu \in \conv\bigl(\eta([b, a]_{\mathbb{T}})\bigr).
    \end{align*}
    Due to strict convexity $\eta(\lambda) \notin [\eta(a), \eta(b)]$ for any $\lambda \neq a, b$, thus $\c_\mu(f) = \c_\mu(g)$ should imply
    \[
    \int_{(a, b)} (f - g)_+ \d \tilde\mu = \int_{(b, a)_{\mathbb{T}}} (f - g)_- \d \tilde\mu = 0.
    \]
    Then (w.l.o.g.) $a \in I$, $b \in [0, 1] \setminus I$, and $\eta(a) = \eta(b)$, what contradicts strict convexity (note that $(a, b) \neq (0, 1)$ because $\inter([0, 1] \setminus I) \neq \emptyset$).
    The claim follows.
\end{proof}

\subsection{Generalized von Mises model}

The next family of functions that we are interested in is the generalized von Mises one defined by \eqref{eq:gen_von_mises}. The next proposition shows that this model is able to cover a color triangle in the case of a convex locus.

\begin{proposition}\label{prop:general_von_Mises}
    Consider the subset of normalized functions from the generalized von Mises family with continuous $1$-periodic function $h \colon \R \to [-\infty, + \infty)$, that has only one maximum point on $\mathbb{T}$:
    \begin{align*}
        \mathcal{F} &\eqset \bigl\{f_{s,a}(\lambda) \eqset \exp\{a h(\lambda - s) + b(s, a)\} : a \ge 0, s \in [0, 1] \bigr\}, \\
        b(s, a) &\eqset - \ln\left(\int_{[0,1]} e^{a h(\lambda - s)} \d \tilde\mu(\lambda)\right).
    \end{align*}
    Let the spectral locus be convex.
    Then $\inter(\TT) \subset \c_\mu(\mathcal{F})$. 
    In particular, $\overline{\c_\mu(\mathcal{F})} = \overline{\TT}$.
    
    Moreover, if $h(\lambda) = \cos (2 \pi \lambda)$ (i.e.\ for the von Mises family) and the spectral locus is strictly convex, then $\c_\mu \colon \mathcal{F} \to \inter(\TT)$ is a bijection.
\end{proposition}

\begin{proof}
    \textbf{Surjectivity.}
    The idea of the proof is the same as in Proposition \ref{prop:step_functions}. 
    First, we are going to show that the function
    \[
    \bm{g}(s, a) \eqset \c_\mu(f_{s, a})
    \]
    is continuous on $[0, 1] \times [0, +\infty)$. Indeed, 
    \[
    \bm{g}(s, a) = \frac{1}{\int_{[0,1]} e^{a h(\lambda - s)} \d \tilde\mu} \int_{[0,1]} \eta(\lambda) e^{a h(\lambda - s)} \d \tilde\mu
    \]
    is also a fraction of continuous functions due to continuity of $h$.
    
    Second, we want to prove that $\bm{g}(s, a)$ converges to the boundary of the color triangle as $a \to \infty$. 
    Let us assume w.l.o.g.\ that $h$ attains it maximum at $0$. Fix $\delta > 0$; then 
    \[
    \max_{\delta \le \lambda \le 1 - \delta} h(\lambda) = h(0) - \eps,\quad \eps > 0.
    \]
    Now take $0 < r \le \delta$ such that $h(\lambda) \ge h(0) - \frac{\eps}{2}$ for $\lambda \in [-r, r]$. Thus for any $a \ge 0$, $s \in [\delta, 1 - \delta]$, and $\lambda \in [s - \delta, s + \delta]$ one has
    \begin{align*}
        e^{-b(s, a)} &= \int_{[0,1]} a^{a h(x - s)} \d \tilde\mu(x) 
        \ge \int_{[s - r, s + r]} e^{a h(x - s)} \d \tilde\mu(x) \\
        &\ge e^{a (h(0) - \eps / 2)} \tilde\mu\bigl([s - r, s + r]\bigr),
    \end{align*}
    and hence
    \begin{align*}
        \norm{\bm{g}(s, a) - \eta(\lambda)}_1 &\le \int_{[0,1]} f_{s,a}(x) \norm{\eta(x) - \eta(\lambda)}_1 \d \tilde\mu(x) \\
        &\le \omega_\eta(2 \delta) \int_{[s - \delta, s + \delta]} f_{s,a} \d \tilde\mu + 
        \int_{(s + \delta, s - \delta)_{\mathbb{T}}} f_{s,a} \d \tilde\mu \\
        &\le \omega_\eta(2 \delta) + e^{a (h(0) - \eps) + b(s, a)} \\
        &\le \omega_\eta(2 \delta) + \frac{e^{- a \eps / 2}}{\tilde\mu\bigl([s - r, s + r]\bigr)}.
    \end{align*}
    Here we also used that $\norm{\eta(x) - \eta(\lambda)}_1 \le 1$. Since $m_r \eqset \inf_{s \in [0, 1]} \tilde\mu\bigl([s - r, s + r]\bigr) > 0$, we get that $s \mapsto \norm{\bm{g}(s, a) - \eta(s)}_1$ converges to $0$ locally uniformly on $(0, 1)$.
    In the same way one can show that
    \begin{align*}
        d\bigl(\bm{g}(s, a), [\eta(0), \eta(1)]\bigr) &\le \omega_\eta(2 \delta) + \frac{e^{- a \eps / 2}}{\tilde\mu\bigl([\{s - r\}, \{s + r\}]_{\mathbb{T}}\bigr)} \\
        &\le \omega_\eta(2 \delta) + \frac{e^{- a \eps / 2}}{m_r}, \quad s \in [1 - \delta, \delta]_{\mathbb{T}}.
    \end{align*}
    Since $\bm{g}$ is periodic in the first argument, the rest of the proof just repeats the arguments from the proof of Proposition~\ref{prop:step_functions}.
    
    \textbf{Bijectivity.}
    It is enough to show that for any $f_{s,a} \neq f_{s',a'} \in \mathcal{F}$ the function $f_{s,a} - f_{s',a'}$ changes sign twice on $[0, 1]$. Clearly, it holds iff $u \eqset \log f_{s,a} - \log f_{s', a'}$ changes sign twice. There exist $A \neq 0$, $\lambda_0$ such that
    \begin{align*}
        u(\lambda) &= a \cos 2 \pi (\lambda - s) + b(s, a) - a' \cos 2 \pi (\lambda - s') - b(s', a') \\
        &= A \cos 2 \pi (\lambda - \lambda_0) + b(s, a) - b(s', a').
    \end{align*}
    Note that since $\int f_{s,a} \d \tilde\mu = \int f_{s',a'} \d \tilde\mu = 1$, we have $\max u > 0$ and $\min u < 0$, thus $u$ changes sign twice, as required.
\end{proof}

\begin{remark}
    Notice that in the general case we cannot show that the strict convexity of the spectral locus implies injectivity, since functions can not satisfy assumptions of Lemma~\ref{lem:injective}.
\end{remark}

\subsection{Reparametrization of spectrum and nonconvex case}\label{subsec:reparam}

Clearly, the parametrization of spectrum by wavelengths in not the only possible: e.g., one can use frequencies instead. 
In general, any strictly monotone continuous transform of $\Lambda$ does not spoil the model assumptions from Section~\ref{sec:model} and preserves (strict) convexity of the spectral locus (cf.\ with a discussion in the end of \S(7) in \cite{weinberg1976geometry}). Obviously, the form of functions in Proposition~\ref{prop:general_von_Mises} can be not preserved under this transform, but the claim of the proposition still holds.
Among all possibilities let us mention a constant speed parametrization given by $f \colon [a, b] \to [\lmin, \lmax]$, such that $\norm*{\frac{d}{d x} \eta\bigl(f(x)\bigr)} \equiv const$ for a.e.\ $x \in [a, b]$ (assuming that $\eta$ is absolutely continuous).

What is more important, is that one can use a reparametrization to ``convexify'' a non-convex spectral locus, like those illustrated on Fig. \ref{fig:colortriangle}. 
\begin{definition}
    We say a spectral locus is piecewise convex if there exist wavelengths $\lmin \le \lambda_1 < \lambda_2 < \dots < \lambda_{2 n} \le \lmax$ such that 
    \begin{equation}\label{eq:T_piecewise_boundary}
        \eta(\Lambda') \subset \partial \TT, \quad \TT \subset \conv\bigl(\eta(\Lambda')\bigr),
    \end{equation}
    where $\Lambda'$ the ``restricted spectrum'' $\bigcup_{k = 1}^n [\lambda_{2 k - 1}, \lambda_{2 k}]$, and for any $\c \in \inter \TT$ the angle from $\c$ to $\eta(\lambda)$ is a monotone function on $\Lambda'$ with 
    \[
    \abs*{\angle(\eta(\lambda_1), \c) - \angle(\eta(\lambda_{2 n}), \c)} \le 2 \pi.
    \]
    
    Respectively, if also $\eta(\lambda) \notin [\eta(a), \eta(b)]$ for any $a \neq b \neq \lambda \in \Lambda'$, then the spectral locus is piecewise strictly convex.
\end{definition}

I.e., instead of the spectral locus and the purple segment $[\eta(\lmin), \eta(\lmax)]$ the boundary $\partial \TT$ consists of $n$ arcs of the locus and $n$ segments. In the above definition one can actually consider a permutation $\pi \in S(n)$ such that $\lmin \le \lambda_{2 \pi(1) - 1} < \lambda_{2 \pi(1)} < \dots < \lambda_{2 \pi(n)} \le \lmax$. However, for the sake of simplicity we stick to $\pi = id$.

Moreover, this approach can be also applied to the case of discontinuous function $\eta$ (what formally is beyond the setting considered in Section~\ref{sec:model}), if there are wavelengths $\lambda_1 < \lambda_2 \le \lambda_3 < \dots \le \lambda_{2 n - 1} < \lambda_{2 n}$ such that $\eta$ can be continuously extended to $[\lambda_{2 k - 1}, \lambda_{2 k}]$ for any $1 \le k \le n$.

Now let us ``glue'' the segments of $\Lambda'$, i.e.\ map it (w.l.o.g.) onto the torus $\mathbb{T}$ by identifying $\lambda_{2 k}$ and $\lambda_{2 k + 1}$, $1 \le k \le n$. Respectively, this induces a measure $\tilde\mu'$ on $\mathbb{T}$.
Then Proposition~\ref{prop:general_von_Mises} has the following counterparts.

\begin{proposition}\label{prop:general_von_Mises_2}
    Let the spectral locus be piecewise convex.
    Consider a generalized von Mises family of functions $\mathcal{F}$ on $\mathbb{T}$ identified with $\Lambda'$ as in Proposition~\ref{prop:general_von_Mises} with $\Delta \lambda = \sum_{i = 1}^n (\lambda_{2 i} - \lambda_{2 i - 1})$; extend them on $\Lambda$ with $0$.
    Then $\inter(\TT) \subset \c_\mu(\mathcal{F})$.
    Moreover, if the spectral locus is piecewise strictly convex, then $\c_\mu(\mathcal{F}) = \inter(\TT)$.
\end{proposition}

\begin{proof}
    The proof is a simple combination of the proofs of Propositions~\ref{prop:general_von_Mises} and~\ref{prop:step_functions_2}.
\end{proof}

%% file: sections/conclusion.tex
\section{Conclusion}\label{sec:theo_concl}

We have shown that the banded model is the only spectral model closed under addition and multiplication (Proposition \ref{prop:banded_unique}), and it has been proven that any matrix with columns from a color cone can be matched with some accuracy with the transition matrix from the set of parameters of the banded model to the set of tristimulus (Proposition \ref{prop:banded_matrix}). For the von Mises model, its uniqueness has been shown in terms of multiplication closure and periodicity properties (Proposition \ref{prop:mises_uniqueness}). In addition, it has also been shown that the Gaussian model is the limiting case of the von Mises model (Proposition \ref{prop:mises_gauss}).

The main result is the fact that the color triangle can be completely covered by the von Mises model in the case of both convex (Proposition \ref{prop:general_von_Mises}) and non-convex spectral locus (Proposition \ref{prop:general_von_Mises_2}).
The latter circumstance is most relevant from the point of view of construction image processing systems for cameras of modern mobile devices.
The result mentioned above is based on Proposition~\ref{prop:step_functions} about covering a color triangle using step-functions, described in the Appendix \ref{sec:appendix}.

\section{Acknowledgements}

The authors express their gratitude to Dr. Veniamin Blinov for careful reading and valuable comments on the text of this paper, as well as to Irina Zhdanova for her help in preparing beautiful illustrations.

%% file: sections/appendix.tex
\section{Step functions}\label{sec:appendix}

\begin{proposition}[Step functions]\label{prop:step_functions}
    Consider the following family of $1$-periodic two-sided step functions on $[0, 1]$:
    \begin{align*}
        \mathcal{F} &\eqset \bigl\{f_{s, \delta} : s \in [0, 1],\; \delta \in (0, 1] \bigr\}, \\
        f_{s, \delta}(\lambda) &\eqset \frac{f_\delta(\lambda - s)}{\int_{[0, 1]} f_\delta(\xi - s) \d \tilde\mu(\xi)}, \\
        f_\delta(\lambda) &\eqset \ind_{[0, \delta]}(\lambda - \lfloor \lambda \rfloor).
    \end{align*}
    Let $\tilde\mu$ be atomless and the spectral locus be convex.
    Then $\inter(\TT) \subset \c_\mu(\mathcal{F})$. In particular, $\overline{\c_\mu(\mathcal{F})} = \overline{\TT}$.
    
    If, in addition, the spectral locus is strictly convex, then $\c_\mu \colon \mathcal{F} \to \inter(\TT)$ is a bijection (note that $f_{s, 1} \equiv 1$ is the same function for any $s$).
\end{proposition}

\begin{proof}
    \textbf{Surjectivity.}
    Since $\tilde\mu$ is atomless, for any $\lambda \in \Lambda$ it holds that
    \[
    \lim_{\delta \to 0} \tilde\mu([\lambda - \delta, \lambda + \delta]) = \mu\left(\{\lambda\}\right) = 0 .
    \]
    Then functions
    \[
    (a, b) \mapsto \tilde\mu([a, b]) \text{ and } (a, b) \mapsto \int_{[a, b]} \eta \d \tilde\mu
    \]
    are continuous. Thus, using $\supp \tilde\mu = [0, 1]$, we obtain that
    \[
    \bm{g}(s, \delta) \eqset \c_\mu(f_{s, \delta}) = \frac{1}{\sum_{n = -1}^0 \tilde\mu([n + s, n + s + \delta])} \sum_{n = -1}^0 \int_{[n + s, n + s + \delta]} \eta \d \tilde\mu
    \]
    is also continuous on $[0, 1] \times (0, 1]$ as a fraction of continuous functions.
    For any $0 < \delta \le 1$ the above function is periodic in first argument: $\bm{g}(1, \delta) = \bm{g}(0, \delta)$.
    For $\delta = 1$ it is equal to a ``white'' color:
    \[
    \bm{g}(s, 1) \equiv \c_w \eqset \c_\mu(f_{0, 1}) = \int_{[0, 1]} \eta \d \tilde\mu .
    \]
    
    Further, since 
    \[
    \bm{g}(s, \delta) = \frac{1}{\tilde\mu([s, s + \delta])} \int_{[s, s + \delta]} \eta \d \tilde\mu \in \conv\bigl(\eta([s, s + \delta])\bigr) 
    \text{ if } s + \delta \le 1,
    \]
    we obtain  that $\norm{\bm{g}(s, \delta) - \eta(\lambda)}_1 \le \omega_\eta(\delta)$ for $0 \le s \le \lambda \le s + \delta \le 1$, where
    \[
    \omega_\eta(\delta) \eqset \max\bigl\{\norm{\eta(a) - \eta(b)}_1 : a, b \in [0, 1] : \abs{a - b} \le \delta \bigr\} \to 0 
    \text{ as } \delta \to 0.
    \]
    due to (uniform) continuity of $\eta$.
    Now consider $s + \delta > 1$. Clearly,
    \begin{align*}
        \bm{g}(s, \delta) &= \frac{1}{\tilde\mu([s, s + \delta - 1]_{\mathbb{T}})} \int_{[0, s + \delta - 1]} \eta \d \tilde\mu + \frac{1}{\tilde\mu([s, s + \delta - 1]_{\mathbb{T}})} \int_{[s, 1]} \eta \d \tilde\mu \\
        &= \frac{\tilde\mu([0, s + \delta - 1])}{\tilde\mu([s, s + \delta - 1]_{\mathbb{T}})} \bm{g}(0, s + \delta - 1) + \frac{\tilde\mu([s, 1])}{\tilde\mu([s, s + \delta - 1]_{\mathbb{T}})} \bm{g}(s, 1 - s),
    \end{align*}
    thus $\bm{g}(s, \delta) \in \left[\bm{g}(0, s + \delta - 1), \bm{g}(s, 1 - s)\right]$.
    Since $s + \delta - 1 \le \delta$ and $1 - s \le \delta$, we obtain
    \[
    d\bigl(\bm{g}(s, \delta), [\eta(0), \eta(1)]\bigr) \le \omega_\eta(\delta), \quad 0 \le s \le 1 < s + \delta.
    \]
    
    Now fix an arbitrary $\c \in \inter(\TT)$. 
    It is clear from the convexity assumption that angle 
    \[
    \theta_0(t) \eqset \begin{cases}
        \angle(\eta(t), \c), & 0 \le t \le 1,\\
        \angle\bigl((t - 1) \eta(0) + (2 - t) \eta(1), \c\bigr), & 1 < t \le 2,
    \end{cases}
    \]
    can be chosen to be continuous and $\abs{\theta_0(2) - \theta_0(0)} = 2 \pi$.
    Take $\delta_0 > 0$ such that $\omega_\eta(\delta_0) \le \frac{1}{2} d(\partial \TT, \c)$;
    hence $d(\partial \TT, \bm{g}(s, \delta)) \le \frac{1}{2} d(\partial \TT, \c)$ for $0 \le s \le 1$, $0 < \delta \le \delta_0$.
    Then angle $\theta_{\delta_0}(s) \eqset \angle(\bm{g}(s, \delta_0), \c)$ is also continuous and $\abs{\theta_{\delta_0}(0) - \theta_{\delta_0}(1)} = 2 \pi$. 
    Suppose $\c \neq \bm{g}(s, \delta)$ for all $s \in [0, 1]$, $\delta \in [\delta_0, 1]$. We are going to use the fact that $\bm{g}(s, \delta)$ is a homotopy between $\bm{g}(\cdot, \delta_0)$ and $\bm{g}(\cdot, 1) \equiv \c_w$ to obtain a contradiction.
    Continuity of $\bm{g}$ yields 
    \[
    \inf\bigl\{\norm{\c -  \bm{g}(s, \delta)}_1 : s \in [0, 1],\; \delta \in [\delta_0, 1]\bigr\} > 0 .
    \]
    Therefore, $(s, \delta) \mapsto \theta_{\delta}(s) \eqset \angle(\bm{g}(s, \delta), \c)$ is jointly continuous, and
    \[
    \abs{\theta_{\delta}(0) - \theta_{\delta}(1)} \equiv 2 \pi.
    \]
    However, this contradicts the fact that $\bm{g}(s, 1) \equiv \c_w$. Hence, there are $s \in [0, 1]$, $\delta \in [\delta_0, 1]$ such that $\c = \bm{g}(s, \delta)$.
    
    \textbf{Bijectivity.}
    Due to the strict convexity of the spectral locus $\bm{g}(s, \delta) \in \inter \TT$.
    Further, it is easy to see that any pair of functions from $\mathcal{F}$ satisfies assumptions of Lemma~\ref{lem:injective}: this follows from the fact that any function of form $f = u_1 \ind_{[a_1, b_1]} - u_2 \ind_{[a_2, b_2]}$ changes sign at most twice on $\R$, i.e.\ there is an interval $I \subset \R$ satisfying $f \ge 0$ ($f \le 0$) on $I$ and $f \le 0$ (resp. $f \ge 0$) on $\R \setminus I$.
    Thus we immediately conclude that $\c_\mu|_\mathcal{F}$ is injective by Lemma~\ref{lem:injective}, and hence bijective thanks to the first claim.
\end{proof}

\begin{appendix_remark}\label{rem:closure}
    Note that taking a closure of the set $\left\{f \tilde \mu : f \in \mathcal{F}\right\}$ in the weak topology on $\MM_+(\Lambda)$ adds to it $\delta$-measures from $[0, 1]$ and ``purple'' measures of form $t \delta_0 + (1 - t) \delta_1$, $0 < t < 1$. 
    Then $\c$ maps this closure exactly onto $\overline{\TT} = \conv\bigl(\eta([0, 1])\bigr)$, and in strictly convex case it is a bijection.
\end{appendix_remark}

Let us also remark that actually, the relation between step functions and strictly convex spectral locus is even deeper: e.g., for a given illuminance they correspond to so called optimal colors, i.e.\ extreme points of the object-color solid \cite{logvinenko2009object}.

\begin{proposition}\label{prop:step_functions_2}
    Let $\tilde\mu$ be atomless and the spectral locus be piecewise convex.
    Consider a family of two-sided step functions as in Proposition~\ref{prop:step_functions}, defined on $\mathbb{T}$ identified with $\Lambda'$; extend them on $\Lambda$ with $0$.
    Then $\inter(\TT) \subset \c_\mu(\mathcal{F})$.
    If, in addition, the spectral locus is piecewise strictly convex, then $\c_\mu \colon \mathcal{F} \to \inter(\TT)$ is a bijection.
\end{proposition}

\begin{proof}
    \textbf{Surjectivity.}
    Here we denote by $\lambda' \in \mathbb{T}$ the image of $\lambda \in \Lambda'$. In particular, $\lambda'_{2 k} = \lambda'_{2 k + 1}$.
    As in the proof of Proposition~\ref{prop:step_functions} define a continuous function
    \[
    \bm{g}(s, \delta) \eqset \c_\mu(f_{s, \delta}).
    \]
    Clearly, $\norm{\bm{g}(s, \delta) - \eta(\lambda)}_1 \le \omega_\eta(\delta)$ once $\lambda'_{2 k - 1} \le s \le \lambda' \le s + \delta \le \lambda'_{2 k}$ for some $k$. Respectively, 
    \[
    d\bigl(\bm{g}(s, \delta), [\eta(\lambda_{2 k}), \eta(\lambda_{2 k + 1})]\bigr) \le \omega_\eta(\delta),
    \]
    once $\lambda'_{2 k - 1} \le s \le \lambda'_{2 k} = \lambda'_{2 k + 1} < s + \delta \le \lambda'_{2 k + 2}$.
    The rest of the proof repeats the proof of Proposition~\ref{prop:step_functions}.
    
    \textbf{Bijectivity.}
    Clearly, due to the piecewise strict convexity of the spectral locus $\c_\mu(f_{s, \delta}) \in \inter \TT$ for any $s \in [0, 1]$ and $\delta > 0$.
    Further, it is easy to see that the statement of Lemma~\ref{lem:injective} holds as well in the case of a piecewise strictly convex spectral locus and functions on $\Lambda'$. Now as in the proof of Proposition~\ref{prop:step_functions} we can show that a difference of two-sided step functions changes sign twice on $\mathbb{T}$ and thus on $\Lambda'$. Hence $\c_\mu|_\mathcal{F}$ is injective, and therefore it is bijective.
\end{proof}

%% file: ms.bbl
\begin{thebibliography}{10}

\bibitem{horn1984exact}
B.~K.~P. Horn, ``Exact reproduction of colored images,'' {\em Computer Vision,
  Graphics, and Image Processing}, vol.~26, no.~2, pp.~135--167, 1984.

\bibitem{brill1978device}
M.~H. Brill, ``A device performing illuminant-invariant assessment of chromatic
  relations,'' {\em Journal of Theoretical Biology}, vol.~71, no.~3,
  pp.~473--478, 1978.

\bibitem{west1979}
G.~West, ``Color perception and the limits of color constancy,'' {\em Journal
  of Mathematical Biology}, vol.~8, no.~1, pp.~47--53, 1979.

\bibitem{qiu2018image}
J.~Qiu, H.~Xu, Z.~Ye, and C.~Diao, ``Image quality degradation of object-color
  metamer mismatching in digital camera color reproduction,'' {\em Appl. Opt.},
  vol.~57, pp.~2851--2860, Apr 2018.

\bibitem{Finlayson1994}
G.~D. Finlayson, M.~S. Drew, and B.~V. Funt, ``Spectral sharpening: sensor
  transformations for improved color constancy,'' {\em J. Opt. Soc. Am. A},
  vol.~11, pp.~1553--1563, May 1994.

\bibitem{otsu2018}
H.~Otsu, M.~Yamamoto, and T.~Hachisuka, ``Reproducing spectral reflectances
  from tristimulus colours,'' {\em Computer Graphics Forum}, vol.~37, no.~6,
  pp.~370--381, 2018.

\bibitem{logvinenko2009object}
A.~D. Logvinenko, ``An object-color space,'' {\em Journal of Vision}, vol.~9,
  pp.~1--23, 10 2009.

\bibitem{logvinenko2013metamer}
A.~D. Logvinenko, B.~V. Funt, and C.~Godau, ``Metamer mismatching,'' {\em IEEE
  Transactions on Image Processing}, vol.~23, no.~1, pp.~34--43, 2014.

\bibitem{Stiles1962}
W.~S. Stiles and G.~W. Wyszecki, ``Counting metameric object colors,'' {\em J.
  Opt. Soc. Am.}, vol.~52, pp.~313--328, Mar 1962.

\bibitem{weinberg1976geometry}
J.~W. Weinberg, ``The geometry of colors,'' {\em General Relativity and
  Gravitation}, vol.~7, no.~1, pp.~135--169, 1976.

\bibitem{nikolayev1985model}
P.~P. Nikolayev, ``Model of the constancy of colour perception for the case of
  continuous spectral functions,'' {\em Biophysics}, vol.~30, no.~1,
  pp.~112--117, 1985.

\bibitem{Griffin2019}
L.~D. Griffin, ``Reconciling the statistics of spectral reflectance and
  colour,'' {\em PLOS One}, vol.~14, pp.~1--24, 11 2019.

\bibitem{nikolaev2007spectral}
D.~P. Nikolaev and P.~P. Nikolayev, ``On spectral models and colour constancy
  clues,'' in {\em 21st European Conference on Modelling and Simulation},
  pp.~318--323, 2007.

\bibitem{nikolaev2008spectral}
P.~P. Nikolayev, S.~M. Karpenko, and D.~P. Nikolaev, ``Spectral models for
  color constancy: Selection rules,'' {\em Proceedings of ISA RAS}, vol.~38,
  pp.~322--335, 2008.

\bibitem{nikolaev2004linear}
D.~P. Nikolaev and P.~P. Nikolayev, ``Linear color segmentation and its
  implementation,'' {\em Computer Vision and Image Understanding}, vol.~94,
  no.~1, pp.~115--139, 2004.

\bibitem{nikolaev2006efficiency}
D.~P. Nikolaev, P.~P. Nikolayev, and V.~P. Bozhkova, ``Efficiency comparison of
  analytical gaussian and linear spectral models in the same colour constancy
  framework,'' {\em International Journal of Simulation: Systems, Science and
  Technology}, vol.~7, no.~3, pp.~21--36, 2006.

\bibitem{gusamutdinova2017verification}
N.~R. Gusamutdinova, E.~I. Ershov, S.~A. Gladilin, and D.~P. Nikolaev,
  ``{Verification of applicability two multiplicative closed spectral models
  for multiple reflection effect description},'' in {\em 2016 International
  Conference on Robotics and Machine Vision} (A.~V. Bernstein, A.~Olaru, and
  J.~Zhou, eds.), vol.~10253, pp.~16--20, International Society for Optics and
  Photonics, SPIE, 2017.

\bibitem{mizokami2012}
Y.~Mizokami and M.~A. Webster, ``Are gaussian spectra a viable perceptual
  assumption in color appearance?,'' {\em J. Opt. Soc. Am. A}, vol.~29,
  pp.~A10--A18, Feb 2012.

\bibitem{mirzaei2014object}
H.~Mirzaei and B.~V. Funt, ``Object-color-signal prediction using wraparound
  gaussian metamers,'' {\em J. Opt. Soc. Am. A}, vol.~31, pp.~1680--1687, Jul
  2014.

\bibitem{yilmaz1962theory}
H.~Yilmaz, {\em Color Vision and a New Approach to General Perception}, vol.~1,
  pp.~126--141.
\newblock Boston, MA: Springer US, 1962.

\bibitem{land1971lightness}
E.~H. Land and J.~J. McCann, ``Lightness and retinex theory,'' {\em J. Opt.
  Soc. Am.}, vol.~61, pp.~1--11, Jan 1971.

\bibitem{nyberg1971II}
N.~D. Nyberg, P.~P. Nikolayev, and M.~M. Bongard, ``On the constancy of
  perception of colouration,'' {\em Biophysics}, vol.~16, no.~6,
  pp.~1052--1063, 1971.

\bibitem{Sallstrom1973}
P.~Sällström, ``Colour and physics: Some remarks concerning the physical
  aspects of human colour vision,'' {\em University of Stockholm Institute of
  Physics Report}, vol.~09, p.~73, 1973.

\bibitem{buchsbaum1980spacial}
G.~Buchsbaum, ``A spatial processor model for object colour perception,'' {\em
  Journal of the Franklin Institute}, vol.~310, no.~1, pp.~1--26, 1980.

\bibitem{cohen1982r_matrix}
J.~B. Cohen and W.~E. Kappauf, ``Metameric color stimuli, fundamental metamers,
  and wyszecki's metameric blacks,'' {\em The American Journal of Psychology},
  vol.~95, no.~4, pp.~537--564, 1982.

\bibitem{Maloney86constancy}
L.~T. Maloney and B.~A. Wandell, ``Color constancy: a method for recovering
  surface spectral reflectance,'' {\em J. Opt. Soc. Am. A}, vol.~3, pp.~29--33,
  Jan 1986.

\bibitem{maloney1986evaluation}
L.~T. Maloney, ``Evaluation of linear models of surface spectral reflectance
  with small numbers of parameters,'' {\em J. Opt. Soc. Am. A}, vol.~3,
  pp.~1673--1683, Oct 1986.

\bibitem{Marimont1992}
D.~H. Marimont and B.~A. Wandell, ``Linear models of surface and illuminant
  spectra,'' {\em J. Opt. Soc. Am. A}, vol.~9, pp.~1905--1913, Nov 1992.

\bibitem{Lee1995}
S.-D. Lee, C.-Y. Kim, and Y.-S. Seo, ``Linear model of surface and scanner
  characterization method,'' in {\em Device-Independent Color Imaging II}
  (E.~Walowit, ed.), vol.~2414, pp.~84 -- 93, International Society for Optics
  and Photonics, SPIE, 1995.

\bibitem{brill1986chromatic}
M.~H. Brill and G.~West, ``Chromatic adaptation and color constancy: A possible
  dichotomy,'' {\em Color Research \& Application}, vol.~11, no.~3,
  pp.~196--204, 1986.

\bibitem{maximov1984transformation}
V.~V. Maximov, {\em Transformatsiya tsveta pri izmenenii osveshcheniya
  (Transformations of colour under illumination changes)}.
\newblock Nauka, 1984.

\bibitem{Brill2002}
M.~H. Brill and G.~D. Finlayson, ``Illuminant invariance from a single
  reflected light,'' {\em Color Research \& Application}, vol.~27, no.~1,
  pp.~45--48, 2002.

\bibitem{MacLeod2003}
D.~I.~A. MacLeod and J.~Golz, {\em A computational analysis of colour
  constancy}.
\newblock R. Mausfeld and D. Heyer, eds., Oxford University, 2003.

\bibitem{logvinenko2013object}
A.~D. Logvinenko, ``Object-colour manifold,'' {\em International Journal of
  Computer Vision}, vol.~101, no.~1, pp.~143--160, 2013.

\bibitem{Mizokami2006}
Y.~Mizokami, J.~S. Werner, M.~A. Crognale, and M.~A. Webster, ``Nonlinearities
  in color coding: Compensating color appearance for the eye's spectral
  sensitivity,'' {\em Journal of Vision}, vol.~6, pp.~996--1007, 08 2006.

\bibitem{nikolaev2005comparative}
D.~P. Nikolaev and P.~P. Nikolayev, ``Comparative analysis of gaussian and
  linear spectral models for colour constancy,'' in {\em Proceedings of 19th
  European Conference on Modelling and Simulation}, pp.~300--305, 2005.

\bibitem{Jiang2013}
J.~Jiang, D.~Liu, J.~Gu, and S.~Süsstrunk, ``What is the space of spectral
  sensitivity functions for digital color cameras?,'' in {\em 2013 IEEE
  Workshop on Applications of Computer Vision (WACV)}, pp.~168--179, 2013.

\bibitem{Shepelev2020}
D.~A. Shepelev, V.~P. Bozhkova, E.~I. Ershov, and D.~P. Nikolaev, ``Simulation
  of underwater color images using banded spectral model,'' in {\em 34th
  International ECMS Conference on Modelling and Simulation}, vol.~34,
  pp.~11--18, 06 2020.

\end{thebibliography}
